\documentclass{article}
\usepackage{iclr2027_conference,times}

\usepackage{amsmath,amsfonts,bm}

\def\eqref#1{equation~\ref{#1}}

\def\1{\bm{1}}

\DeclareMathAlphabet{\mathsfit}{\encodingdefault}{\sfdefault}{m}{sl}
\SetMathAlphabet{\mathsfit}{bold}{\encodingdefault}{\sfdefault}{bx}{n}

\usepackage{hyperref}
\hypersetup{hidelinks}
\usepackage{url}
\usepackage{graphicx}
\usepackage{booktabs}
\usepackage{amsmath}
\usepackage{amssymb}
\usepackage{array}
\usepackage{xspace}
\usepackage{algorithm}
\usepackage{algpseudocode}
\usepackage{multirow}
\newtheorem{proposition}{Proposition}

\graphicspath{{figures/}}

\title{BRIDGE: Selective Boundary Repair with Dependency Graph Evidence in Multi-Agent Systems}

\iclrfinalcopy

\author{
\makebox[\textwidth][c]{%
\begin{tabular}{c}
Encheng Su\textsuperscript{1, 10, *}, Pengze Li\textsuperscript{2, 10,*}, Ziheng Hua\textsuperscript{2,*}, Qiucheng Yu\textsuperscript{3}, Jianyu Wu\textsuperscript{4, 10}\\
Chen Tang\textsuperscript{5}, Yizhou Wang\textsuperscript{5}, Lintao Wang\textsuperscript{6}, Junjia Xiang\textsuperscript{7}, Xinzhu Ma\textsuperscript{8}\\
Xi Chen\textsuperscript{9}, Shixiang Tang\textsuperscript{5}, Houqiang Li\textsuperscript{1}, Xi Chen\textsuperscript{2\textdagger}, Aoran Wang\textsuperscript{10\textdagger}\\[2pt]
\small
\textsuperscript{1}University of Science and Technology of China \quad \textsuperscript{2}Fudan University \quad \\
\small
\textsuperscript{3}City University of Hong Kong \quad 
\textsuperscript{4}Shanghai Jiao Tong University \quad \\
\small
\textsuperscript{5}The Chinese University of Hong Kong
\textsuperscript{6}The University of Sydney \quad \\
\small
\textsuperscript{7}Harbin Institute of Technology \quad \textsuperscript{8}Beihang University\\
\small
\textsuperscript{9}East China Normal University \quad \textsuperscript{10}Shanghai Artificial Intelligence Laboratory\\ \quad 
\normalfont
\textsuperscript{*}Equal contribution
\quad
\textsuperscript{\textdagger}Corresponding authors\quad 
\end{tabular}
}
}

\newcommand{\method}{BRIDGE\xspace}
\newcommand{\dataset}{HandoffProbe\xspace}

\begin{document}

\maketitle
\pagestyle{empty}      
\thispagestyle{empty}  

\begin{abstract}
Multi-agent collaboration has emerged as a critical paradigm for tackling complex, long-horizon tasks; however, while coordination protocols continue to advance, systemic failures at the inter-agent collaboration interface remain a pervasive bottleneck. In collaborative workflows, systems frequently suffer from \emph{broken handoffs}---where critical task context is omitted, misrouted, or ignored despite each agent emitting locally plausible responses. To address this challenge, we present \method, a verifiable runtime control framework that treats handoff repair as selective intervention over explicit dependency graphs. Central to \method is an integrated fault-attribution and calibration pipeline, \dataset, which decouples execution dependencies and injects counterfactual handoff perturbations to provide empirical supervision for runtime intervention. Leveraging this diagnostic foundation, \method dynamically compiles observable interaction traces into a provenance graph to pinpoint unresolved dependencies, without accessing private agent states. It then evaluates candidate repairs through an empirical gate that weighs expected gains against disruption risks, and executes repairs only when formal local postconditions hold. Extensive evaluations across demanding benchmarks demonstrate that \method boosts task success by up to 12.3 absolute percentage points on SILO-BENCH and 11.1 on TeamBench. Across both backbones, \method maintains strict cost-effectiveness, adding as little as 2.0\% marginal token overhead. Furthermore, on 562 handoff-relevant traces from the MAST benchmark, \method achieves 69.8\% F1 in a single pass, establishing a reliable operational boundary for verifiable coordination repair.
\end{abstract}


\section{Introduction}
\label{sec:intro}

Large language model (LLM) architectures are increasingly deployed as multi-agent collectives where specialized agents assume complementary roles, such as planners, executors, and verifiers~\citep{hong2024metagpt,wu2024autogen,li2023camel}. In these collaborative workflows, task execution fundamentally relies on directional information flow across agent interfaces, where the synthesized output of an upstream agent constitutes the execution context of a downstream recipient. We define the transfer of such task-critical state representations and intermediate outputs as an \emph{inter-agent handoff}. In practice, multi-agent collectives frequently succumb to systemic \emph{broken handoffs}: critical dependency constraints are dropped, misrouted, or received but ignored by downstream agents. Consequently, even when every individual agent exhibits robust local competence and emits plausible responses, the multi-agent collective remains acutely vulnerable to communicative breakdowns at the collaboration boundary.

Existing approaches attempt to mitigate coordination errors through task decomposition and allocation~\citep{li2025aop}, multi-order communication and message consolidation~\citep{guan2026moc}, or post-hoc trace attribution~\citep{cemri2025mast,zhang2025whowhen}. However, static protocols cannot accommodate dynamic runtime failures, while post-mortem diagnostics arrive too late to prevent cascading errors. When attempting runtime remediation, conventional strategies such as unconstrained query clarification~\citep{lin2026agentask} and unguided self-refinement~\citep{madaan2023selfrefine} introduce severe side effects: an upstream omission may be non-essential to downstream success, and unwarranted re-prompting frequently corrupts valid intermediate trajectory states while multiplying latency and token overhead. Fundamentally, treating runtime intervention as an unguided generative task fails because it lacks two guarantees: empirical awareness of whether an intervention provides net benefit, and verifiable assurance that the missing dependency has actually been supplied.

To resolve this dilemma, we present an integrated framework that decouples offline empirical calibration from online verifiable regulation: \textbf{\dataset} and \textbf{\method}. Because live executions provide no immediate signals regarding whether an intervention will help or harm, \dataset first acts as an offline simulation engine. As illustrated in Figure~\ref{fig:overview}, \dataset systematically perturbs historical execution traces, injects typical communication faults, and compares end-to-end task outcomes with and without intervention to learn an empirical prior of when an intervention is genuinely beneficial. At runtime, \method acts as the online controller. It dynamically translates public agent interactions into an explicit \emph{provenance graph} that exposes missing dependencies as visible gaps. Before taking action, \method consults the empirical priors learned by \dataset via a calibrated utility gate to decide whether intervention is warranted. If approved, \method verifies that the generated repair successfully resolves the missing dependency before committing the updated message downstream; otherwise, it safely rolls back to preserve the original state.

Extensive evaluations confirm that our framework establishes new state-of-the-art performance across demanding multi-agent benchmarks, including SILO-BENCH and TeamBench~\citep{zhang2026silo,kim2026teambench}. On SILO-BENCH, our system boosts task success from 63.3\% to 75.6\% with DeepSeek-V4.1-flash and from 68.9\% to 77.8\% with GPT-5.6-Terra. On TeamBench, it improves success rates by up to 11.1 absolute percentage points. Crucially, these gains are achieved with minimal computational footprint---incurring only 2.0\% marginal token overhead on GPT-5.6-Terra and an average of 5.3\% on TeamBench. Furthermore, on 562 handoff-relevant traces from the MAST benchmark, our single-call graph diagnosis achieves 69.8\% F1 on GPT-5.6-Terra while saving 32\%--42\% in token costs compared to multi-stage baselines.

Our core contributions are threefold:
\begin{itemize}
    \item We formalize multi-agent handoff repair as a verifiable state-regulation problem over an explicit provenance graph. By tracking directional information dependencies, \method transforms subtle conversational failures into observable graph gaps, cleanly decoupling structural dependency tracking from semantic generation.
    \item We develop \dataset, an offline simulation pipeline that pairs counterfactual trajectory perturbations to learn empirical intervention priors. Our analysis reveals two key insights: 87.5\% of unselected repairs waste tokens without improving task success, and interventions succeed only when strictly aligned with designated failure modes---mismatches often corrupt valid intermediate states.
    \item We introduce a transactional execution mechanism that explicitly checks whether intended repairs have taken effect, guaranteeing that unverified modifications roll back safely without corrupting execution states. \method achieves new state-of-the-art performance across diverse benchmarks with negligible token overhead (as low as 2.0\%), delivering both high reliability and efficiency.
\end{itemize}

\begin{figure}[t]
\centering
\includegraphics[width=\linewidth]{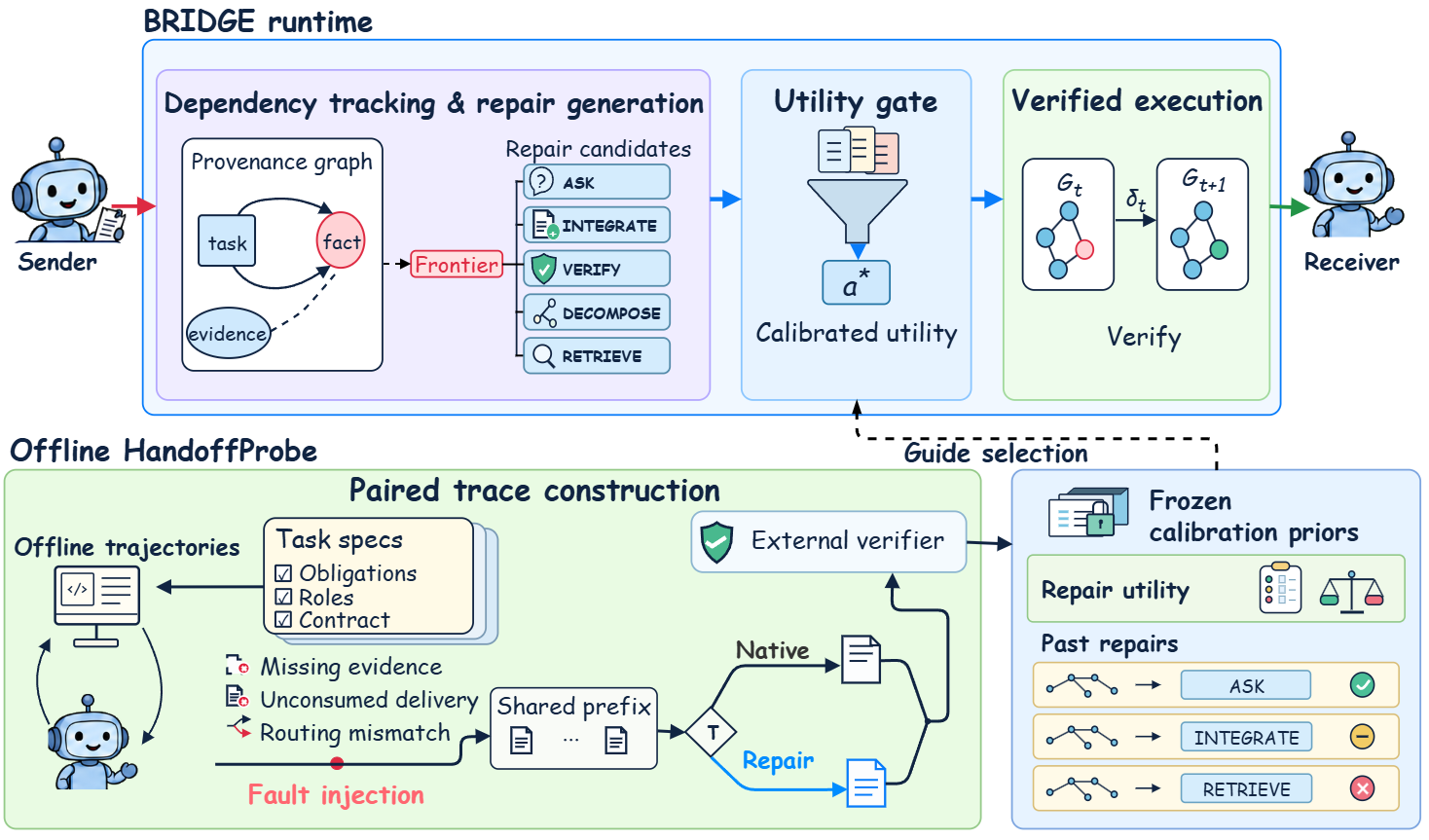}
\caption{Overview of \dataset calibration and \method runtime control. \dataset injects handoff faults into offline trajectories and branches a shared faulted prefix into native and repaired continuations. An external verifier evaluates final task success for each continuation, and the paired outcomes are aggregated into frozen calibration priors. At runtime, \method constructs a provenance graph, identifies unresolved frontiers, and uses the frozen priors to gate candidate repairs.}
\label{fig:overview}
\end{figure}


\section{Related Work}
\label{sec:related}

\paragraph{Multi-Agent Coordination and Communication.}
Multi-agent systems coordinate role-specialized agents through task decomposition and message passing. Recent efforts predominantly focus on structuring task allocation and communication topologies. For instance, Agent-Oriented Planning refines subtask decomposition to balance workloads~\citep{li2025aop}, AgentPrune prunes redundant communication links to reduce message noise~\citep{zhang2024agentprune}, MOC aggregates multi-hop interactions into structured contexts~\citep{guan2026moc}, and Routed Graph Handoff dynamically switches between structured dependency graphs and free-form text messages~\citep{banerjee2026routed}. While these methods optimize the routing paths of inter-agent messages, they assume the underlying communication is intrinsically dependable and cannot track at runtime whether delegated information is correctly consumed by downstream agents.

\paragraph{Failure Attribution and Robustness Evaluation.}
Execution-trace analyses demonstrate that multi-agent failures typically stem from coordination misalignments that cascade across agents. MAST provides a taxonomy of inter-agent misalignment and verification failures~\citep{cemri2025mast}; Who\&When formulates automated attribution to pinpoint which agent and step triggered a breakdown~\citep{zhang2025whowhen}; and MAS-FIRE evaluates architecture robustness by deliberately injecting prompt and routing faults~\citep{jia2026masfire}. However, these approaches remain diagnostic or benchmark-oriented. In contrast to passive post-mortem attribution, our \dataset module actively simulates handoff faults on execution traces and contrasts task outcomes, directly measuring the empirical utility of candidate interventions to guide proactive runtime repair.

\paragraph{Runtime Intervention, Verification, and Selective Decisions.}
To prevent early errors from propagating across agents, recent studies introduce dynamic interventions during execution. Common strategies rely on self-refinement loops~\citep{madaan2023selfrefine}, query clarification~\citep{lin2026agentask}, or predefined validation functions over intermediate subtasks~\citep{xu2026verimap}. In parallel, experience-reuse methods retrieve past workflows and interaction patterns to guide agent decisions~\citep{wang2023voyager,wang2024awm,zhang2025gmemory}. Structured memory systems such as S3Mem organize long interaction histories into scene–event units and route compact query-relevant evidence to a frozen reader~\citep{su2026s3mem}. However, existing intervention methods often intervene indiscriminately, risking disruptions to valid executions, and rarely check whether a modification actually resolves the issue. Inspired by selective decision frameworks~\citep{geifman2019selectivenet}, BRIDGE intervenes only when a repair is estimated to be genuinely helpful and releases the modified message only after confirming that the missing information is in place, avoiding unneeded interference and degraded performance.


\section{Method}
\label{sec:method}

\method treats handoff repair as a runtime decision problem with explicit verification. As illustrated in Figure~\ref{fig:overview}, it separates offline utility calibration (\dataset) from online execution. At runtime, the controller tracks execution dependencies, generates feasible repair actions, filters them using calibrated empirical estimates, and commits a repair only after verifying its graph-level postcondition.

\subsection{Problem Formulation}
\label{sec:problem}
Consider a multi-agent system where role-specialized agents collaborate via messages and tool invocations governed by a shared execution contract $(O, R, \mathcal V)$, specifying obligations $O$, assigned roles $R$, and operational constraints $\mathcal V$. At any agent handoff, the controller observes the public execution history $e_{1:T}$. The objective is to decide whether to forward the native handoff message $m^{\mathrm{nat}}$ or replace it with a verified repair action $a$. 

Each candidate repair $a$ is constrained by an execution precondition $\mathrm{pre}(a)$ and an intended postcondition $\mathrm{post}(a)$. The controller first filters candidates by structural feasibility and then selects an action whose expected benefit justifies the risk of disruption. Figure~\ref{fig:overview_workflow} illustrates this process for a missing-evidence frontier.

\begin{figure}[t]
\centering
\includegraphics[width=\linewidth]{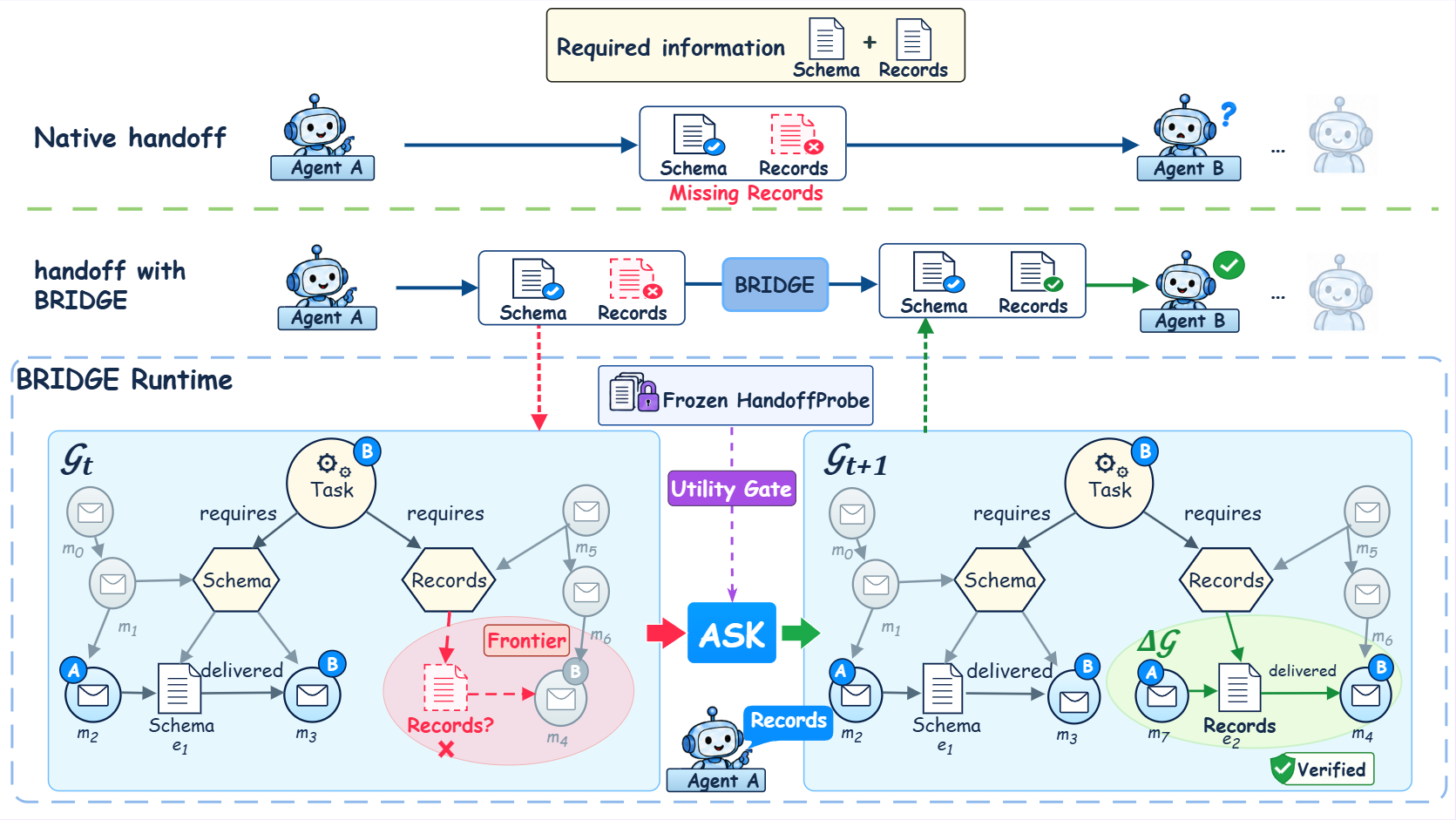}
\caption{Illustration of a missing-evidence repair. Agent A omits Records from its handoff to Agent B; \method retrieves the missing item via \texttt{ASK} and releases the revised handoff only after verifying its delivery in the provenance graph.}
\label{fig:overview_workflow}
\end{figure}

\subsection{Frontier Detection and Candidate Generation}
\label{sec:frontier_generation}
To locate coordination breakdowns, \method maintains a dynamic provenance graph $\mathcal G_t=(\mathcal N_t,\mathcal E_t)$, where nodes represent messages, tool observations, and contract obligations, and directed edges denote data production, routing, and consumption.

Let $q = (\mathrm{id}, \text{desc}, \mathrm{owner}) \in \mathcal Q(O)$ denote a contract-required fact specification, where $\mathrm{owner} \in R$ indicates the producing role. The controller tracks delivered facts $\mathcal D_{\mathrm{delivered}}$, consumed facts $\mathcal D_{\mathrm{consumed}}$, and observed public traces $\mathcal D_{\mathrm{observed}}$. A communication gap emerges when a required fact is unmet, defining the frontier set $\mathcal F = \mathcal F_{\mathrm{use}} \cup \mathcal F_{\mathrm{abs}} \cup \mathcal F_{\mathrm{scope}}$:
\begin{equation}
q \in \begin{cases}
\mathcal F_{\mathrm{use}}, & q \in \mathcal D_{\mathrm{delivered}} \setminus \mathcal D_{\mathrm{consumed}} \quad \text{(unconsumed delivery)}, \\
\mathcal F_{\mathrm{abs}}, & q \notin \mathcal D_{\mathrm{observed}} \quad \text{(missing evidence)}, \\
\mathcal F_{\mathrm{scope}}, & q \in \mathcal D_{\mathrm{delivered}}(r) \text{ with } r \ne \mathrm{dest}(q) \quad \text{(routing mismatch)}.
\end{cases}
\label{eq:frontiers}
\end{equation}
The destination $\mathrm{dest}(q)$ is dynamically resolved from the contract and role assignments. These frontier conditions need not be mutually exclusive; the controller prioritizes the earliest unmet dependency under the graph partial order, $f^* = \min_{\prec} \mathcal F$. Ties among concurrent minimal frontiers are resolved by contract obligation order and causal graph depth.

\begin{table}[!t]
\caption{Domain-agnostic repair actions, prerequisites, and expected state changes.}
\label{tab:actions}
\centering\small
\setlength{\tabcolsep}{5pt}
\renewcommand{\arraystretch}{1.1}
\begin{tabular}{lp{0.36\linewidth}p{0.37\linewidth}}
\toprule
Action Primitive & Prerequisite Condition & Intended State Change \\
\midrule
\texttt{ASK} & Identified owner and valid query & Target evidence acquired from sender \\
\texttt{INTEGRATE} & Available evidence in recipient scope & Target evidence woven into outgoing message \\
\texttt{VERIFY} & Declared contract constraint & Verification certificate recorded in graph \\
\texttt{DECOMPOSE} & Complex unmet obligation & Message rewritten using visible local sources \\
\texttt{RETRIEVE} & Accessible local or public state & Missing context fetched into recipient scope \\
\bottomrule
\end{tabular}
\end{table}

To resolve $f^*$, \method instantiates candidate repairs from a compact set of domain-agnostic primitives defined over communication dependencies and accessible execution states (Table~\ref{tab:actions}). A candidate $a$ is eligible only when its structural preconditions and executing role capabilities hold:
\begin{equation}
\begin{aligned}
\mathrm{eligible}(a) \Longleftrightarrow \; & \mathrm{pre}(a) \subseteq \mathrm{facts}(\mathcal G_t) \\
& \wedge \; \mathrm{cap}_{\mathrm{req}}(a) \subseteq \mathrm{cap}(\mathrm{agent}).
\end{aligned}
\label{eq:valid}
\end{equation}

If missing information exists solely within upstream internal states, \texttt{RETRIEVE} is ineligible, forcing communicative resolution via \texttt{ASK}.

\subsection{\dataset: Offline Utility Calibration}
\label{sec:dataset}
Feasibility confirms structural admissibility but does not guarantee downstream utility. To avoid disruptive interruptions, \dataset provides empirical statistics derived from controlled counterfactual paired rollouts.

\paragraph{Paired Rollouts.}
At an observed handoff boundary $\tau$ with execution history prefix $h_\tau$, benchmark trajectories introduce controlled communication faults. Execution then evaluates two parallel continuations:
\begin{equation}
x^{\mathrm{branch}} = \operatorname{exec}\bigl(h_\tau \mid \text{branch}\bigr), \quad \text{for } \text{branch} \in \{\mathrm{nat}, \mathrm{rep}\},
\label{eq:rollouts}
\end{equation}
where $x^{\mathrm{nat}}$ proceeds without intervention and $x^{\mathrm{rep}}$ executes candidate repair $a$. Both rollouts run to their terminal execution states.

\paragraph{Deterministic Outcome Attribution.}
Task outcomes are verified via deterministic environment oracles $V(x) \in \{0, 1\}$ (such as official test suites and exit codes), avoiding LLM judge variance. Comparing outcomes attributes each intervention to an empirical label $\ell$:
\begin{equation}
\ell = \begin{cases}
\mathrm{helped}, & V(x^{\mathrm{rep}}) = 1 \wedge V(x^{\mathrm{nat}}) = 0, \\
\mathrm{harmed}, & V(x^{\mathrm{rep}}) = 0 \wedge V(x^{\mathrm{nat}}) = 1, \\
\mathrm{neutral}, & V(x^{\mathrm{rep}}) = V(x^{\mathrm{nat}}).
\end{cases}
\label{eq:label}
\end{equation}
Aggregating these outcomes over partitioned action--frontier pairs yields empirical benefit ($r_{\mathrm{help}}$), disruption risk ($r_{\mathrm{harm}}$), and neutrality ($r_{\mathrm{neu}}$) rates, alongside mean token cost ($\bar{c}$).

\subsection{Utility Gate Selection}
\label{sec:gate}
At runtime, the selection policy $\pi$ balances expected success against execution overhead. The base utility of an action--frontier pair is:
\begin{equation}
u(a, f^*) = w_h r_{\mathrm{help}} - w_n r_{\mathrm{neu}} - w_x r_{\mathrm{harm}} - w_c \bar{c},
\label{eq:utility}
\end{equation}
where $w_x$ penalizes disruptive failures, $w_c$ weights token overhead, and $w_n$ accounts for multi-agent latency and state perturbation in neutral executions.

A candidate repair is admitted if and only if:
\begin{equation}
\pi(a, f^*) = \mathbf{1}\!\left[u(a, f^*) > 0 \;\wedge\; r_{\mathrm{harm}} \le \epsilon \;\wedge\; |\mathcal D_{a, f^*}| \ge m_{\min}\right],
\label{eq:admission}
\end{equation}
requiring positive expected utility under bounded empirical disruption ($\le \epsilon$) and sufficient historical support ($m_{\min}$). If no candidate clears this gate, the controller bypasses intervention and forwards the native message (Algorithm~\ref{alg:runtime}, line~8).

\subsection{Verified Execution and Rollback}
\label{sec:transactor}
An admitted repair may still fail at generation time. To prevent unverified modifications from reaching downstream agents, \method executes repairs behind a verify-before-commit barrier:
\begin{equation}
\mathrm{commit} \Longleftrightarrow \operatorname{Check}(\mathrm{post}(a^*), \mathcal G_t, \delta_t) = 1.
\label{eq:commit}
\end{equation}
Here, $\operatorname{Check}(\cdot)$ checks whether the observed transition $\delta_t = \operatorname{Diff}(\mathcal G_t, \mathcal G_{t+1})$ satisfies the formal postcondition $\mathrm{post}(a^*)$, resolves the target frontier, and leaves existing dependency edges valid. 

As formalized in Algorithm~\ref{alg:runtime}, the repair is committed downstream only if verification passes; otherwise, the controller restores the graph snapshot and releases the native message $m^{\mathrm{nat}}$. Rollback applies strictly to the controller's local graph and unreleased message buffer.
\textbf{Execution Mechanism.} Graph tracking, verification, and structural repairs (\texttt{INTEGRATE}, \texttt{RETRIEVE}) are executed programmatically by the controller at zero LLM inference cost. Only communicative repairs (\texttt{ASK}) invoke an auxiliary LLM query to the sender.

\begin{algorithm}[!t]
\caption{Runtime Decision and Verification Loop}
\label{alg:runtime}
\begin{algorithmic}[1]
\Require Handoff message $m^{\mathrm{nat}}$, history $e_{1:T}$, contract $(O, R, \mathcal V)$, priors $\mathcal D_{\mathrm{calib}}$, budget $B$
\Ensure Released handoff message $m^*$
\State $\mathcal G_t \gets \textsc{UpdateGraph}(e_{1:T}, O, R, \mathcal V)$
\State $\mathcal F_t \gets \textsc{DiagnoseFrontiers}(\mathcal G_t)$ \Comment{Eq.~(\ref{eq:frontiers})}
\If{$\mathcal F_t = \emptyset$} \Return $m^{\mathrm{nat}}$ \EndIf
\State $f^* \gets \min_\prec \mathcal F_t$
\State $\mathcal A_{\mathrm{elig}} \gets \{a \mid \mathrm{eligible}(a) \wedge c(a) \le B\}$ \Comment{Eq.~(\ref{eq:valid})}
\State $\mathcal A_{\mathrm{adm}} \gets \{a \in \mathcal A_{\mathrm{elig}} \mid \pi(a, f^*) = 1\}$ \Comment{Eq.~(\ref{eq:admission})}
\If{$\mathcal A_{\mathrm{adm}} = \emptyset$} \Return $m^{\mathrm{nat}}$ \EndIf
\State Select optimal repair $a^* \gets \arg\max_{a \in \mathcal A_{\mathrm{adm}}} u(a, f^*)$
\State Checkpoint state snapshot $(\mathcal G_t, m^{\mathrm{nat}})$
\State $m^{\mathrm{rep}}, \mathcal G_{t+1} \gets \textsc{ExecuteRepair}(a^*)$
\State $\delta_t \gets \operatorname{Diff}(\mathcal G_t, \mathcal G_{t+1})$
\If{$\operatorname{Check}(\mathrm{post}(a^*), \mathcal G_t, \delta_t) = 1$} \Comment{Eq.~(\ref{eq:commit})}
    \State \Return $m^{\mathrm{rep}}$
\Else
    \State Restore graph snapshot $\mathcal G_t$; \Return $m^{\mathrm{nat}}$
\EndIf
\end{algorithmic}
\end{algorithm}


\section{Experiments}
\label{sec:experiments}
We evaluate \method across end-to-end multi-agent benchmarks, real-world execution traces, and counterfactual calibration corpora. Our experiments address four central questions: 
(1) Does \method consistently improve end-to-end task completion while maintaining minimal resource overhead? 
(2) Do calibrated priors successfully transfer zero-shot across environments? 
(3) Can graph-grounded dependency tracking accurately detect handoff failures on real execution traces under strict single-call constraints? 
(4) Why are empirical gating and transactional verification indispensable for safe runtime intervention?

\subsection{Experimental Setup}
\label{sec:exp-setup}

\paragraph{Benchmarks and Model Backbones.}
We evaluate runtime performance on two multi-agent suites: \textbf{SILO-BENCH}~\citep{zhang2026silo} and \textbf{TeamBench}~\citep{kim2026teambench}. SILO-BENCH assesses collaborative information exchange over independent rollouts, while TeamBench tests workflows requiring tight role coordination and multi-step reasoning. To examine cross-model generalizability, all experiments use two state-of-the-art LLM backbones: DeepSeek-V4.1-flash and GPT-5.6-Terra.

\paragraph{Comparative Baselines.}
We compare \method against unassisted \textbf{Native} execution and four multi-agent coordination paradigms. Specifically, we evaluate communication-oriented approaches: \textbf{AgentAsk}~\citep{lin2026agentask}, which queries peers for clarification during ambiguous states, and \textbf{Self-Refine}~\citep{madaan2023selfrefine}, which prompts agents to iteratively refine handoff messages. Furthermore, we benchmark against structured coordination and verification systems: \textbf{MOC}~\citep{guan2026moc}, which aggregates multi-hop interactions into structured summaries, and \textbf{VeriMAP}~\citep{xu2026verimap}, which validates intermediate subgoals against explicit formal constraints.

\paragraph{Evaluation Metrics and Token Accounting.}
Primary task performance is evaluated via \textbf{Full Success / Pass Rate} (percentage of trajectories meeting all verification conditions) and \textbf{Partial Score} (average progress credit). For rigorous efficiency benchmarking, comprehensive token accounting tracks end-to-end prompt and completion tokens across all agents and controllers throughout the trajectory. For diagnostic evaluations, we record per-trace token footprints and model invocations per pass.

\paragraph{Diagnostic Traces on MAST.}
To assess real-world fault diagnosis prior to repair, we evaluate methods on a curated, handoff-focused slice of 562 traces from the \textbf{MAST} corpus~\citep{cemri2025mast} (detailed in Appendix~\ref{app:mast-construction}).

\subsection{End-to-End Task Success and Resource Efficiency}
\label{sec:main-results}

\begin{table}[!htbp]
\caption{End-to-end collaborative performance on TeamBench and SILO-BENCH (\%). Success denotes full task acceptance; Partial denotes benchmark progress score. Best results are in \textbf{bold}; second-best results are \underline{underlined}.}
\label{tab:main}
\centering\footnotesize
\setlength{\tabcolsep}{5pt}
\renewcommand{\arraystretch}{1.12}
\begin{tabular*}{\linewidth}{@{\extracolsep{\fill}}lcccc@{}}
\toprule
 & \multicolumn{2}{c}{TeamBench} & \multicolumn{2}{c}{SILO-BENCH} \\
\cmidrule(lr){2-3}\cmidrule(l){4-5}
Method & Full Pass (\%) $\uparrow$ & Partial (\%) $\uparrow$ & Full Success (\%) $\uparrow$ & Partial (\%) $\uparrow$ \\
\midrule
\multicolumn{5}{l}{\emph{DeepSeek-V4.1-flash}} \\
Native & 48.9 & 81.3 & $63.3\pm3.3$ & 65.8 \\
AgentAsk & 53.3 & 84.3 & $61.1\pm3.3$ & 78.9 \\
Self-Refine & 51.1 & 83.7 & $63.3\pm1.9$ & 78.9 \\
MOC & \underline{54.4} & \underline{84.4} & $\underline{65.6}\pm1.9$ & \underline{82.2} \\
VeriMAP & 51.1 & 83.6 & $61.1\pm3.3$ & 78.9 \\
\textbf{\method (Ours)} & \textbf{60.0} & \textbf{87.8} & $\textbf{75.6}\pm3.8$ & \textbf{83.4} \\
\midrule
\multicolumn{5}{l}{\emph{GPT-5.6-Terra}} \\
Native & 46.7 & 79.6 & $\underline{68.9}\pm4.0$ & 75.3 \\
AgentAsk & \underline{48.9} & 80.9 & $65.6\pm3.3$ & 81.1 \\
Self-Refine & 47.8 & 79.6 & $64.4\pm1.9$ & 80.6 \\
MOC & \underline{48.9} & \underline{81.3} & $63.3\pm1.9$ & 78.9 \\
VeriMAP & 47.8 & 79.7 & $\underline{68.9}\pm0.0$ & \underline{83.3} \\
\textbf{\method (Ours)} & \textbf{52.2} & \textbf{84.7} & $\textbf{77.8}\pm1.1$ & \textbf{87.5} \\
\bottomrule
\end{tabular*}
\end{table}

Table~\ref{tab:main} presents end-to-end task success. \method establishes new state-of-the-art results across both benchmarks and model backbones, outperforming unassisted Native execution by 5.5 to 12.3 percentage points. On SILO-BENCH, \method achieves substantial, statistically significant margins over Native: +12.3 percentage points on DeepSeek-V4.1-flash ($p=0.027$, 16 wins vs.\ 5 losses) and +8.9 percentage points on GPT-5.6-Terra ($p=0.0007$, 36 wins vs.\ 12 losses). In contrast, standard intervention baselines (e.g., Self-Refine and VeriMAP) frequently fail to surpass the Native baseline, as unconstrained interventions introduce spurious clarifications that derail valid multi-agent trajectories.

\paragraph{Runtime Token Efficiency.}
While baselines trigger severe token inflation, \method achieves its state-of-the-art gains with minimal resource footprint: adding merely \textbf{+2.0\%} marginal token overhead on GPT-5.6-Terra (25,693 vs.\ 25,198 tokens) and \textbf{+5.3\%} on TeamBench. By selectively abstaining from unhelpful repairs and rolling back unverified modifications, \method prevents runaway intervention cascades (detailed token waterfall in Appendix~\ref{app:token-overhead}).

\paragraph{Zero-Shot Cross-Benchmark Policy Transfer.}
The TeamBench results further validate the generalizability of our calibrated intervention policy. TeamBench executions directly reuse the offline calibration priors learned from synthetic corpora without exposing a single TeamBench instance during calibration. The resulting lift from 48.9\% to 60.0\% demonstrates that \method learns coordination principles that transfer zero-shot across environments. Conversely, calibration priors do not transfer across disparate model backbones, confirming that behavioral calibration must remain backbone-specific.

\subsection{Handoff Diagnosis on Real Execution Traces}
\label{sec:diagnostic-results}
Accurate runtime repair fundamentally depends on timely and sensitive upstream fault detection. We benchmark \method against Direct Prompting and adapted baselines across 562 handoff-relevant MAST execution traces.

\begin{table}[!htbp]
\caption{Diagnostic performance on 562 handoff-relevant MAST execution traces (\%). All reference baselines require two-stage prompt chains; \method operates strictly via a single diagnostic pass. Best results are in \textbf{bold}; second-best results are \underline{underlined}.}
\label{tab:mast}
\centering\small
\setlength{\tabcolsep}{8pt}
\renewcommand{\arraystretch}{1.15}
\begin{tabular}{lcccc}
\toprule
& \multicolumn{2}{c}{\emph{DeepSeek-V4.1-flash}} & \multicolumn{2}{c}{\emph{GPT-5.6-Terra}} \\
\cmidrule(lr){2-3} \cmidrule(lr){4-5}
Method & F1 (\%) $\uparrow$ & Tokens/trace $\downarrow$ & F1 (\%) $\uparrow$ & Tokens/trace $\downarrow$ \\
\midrule
Direct Prompting        & 62.5          & \textbf{6,686}    & 67.5          & \textbf{6,488} \\
AgentAsk                & 68.4          & 14,268            & 68.4          & 13,982         \\
Self-Refine             & 67.4          & 13,262            & 67.9          & 12,359         \\
MOC                     & \textbf{68.7} & 13,738            & 68.6          & 13,507         \\
VeriMAP                 & 64.2          & 13,478            & \underline{69.4} & 14,917      \\
\textbf{\method (Ours)} & \underline{68.6} & \underline{9,116} & \textbf{69.8} & \underline{8,664} \\
\bottomrule
\end{tabular}
\end{table}

\paragraph{Diagnostic Accuracy and Efficiency.}
As shown in Table~\ref{tab:mast}, \method achieves Pareto-optimal diagnostic performance across backbones. On GPT-5.6-Terra, \method attains the highest F1 score (\textbf{69.8\%}), outperforming competitive two-stage baselines like VeriMAP and MOC. On DeepSeek-V4.1-flash, it reaches a near-optimal \textbf{68.6\%} F1 (within 0.1\% of MOC). Crucially, by operating in a single diagnostic pass rather than multi-stage prompt chains, \method slashes token consumption by \textbf{32\%--42\%} (detailed breakdown and noisy-control specificity in Appendix~\ref{app:mast}).

\subsection{Ablation Studies}
\label{sec:ablation}

\paragraph{Component Necessity and Gating Dynamics.}
We ablate \method on SILO-BENCH (DeepSeek-V4.1-flash). As Table~\ref{tab:ablation} shows, structured provenance, empirical gating, and transactional verification are all indispensable.

Crucially, replacing graph provenance with a sequential text checklist (\emph{w/ Flat Checklist}) drops task success from $75.6\%$ to $68.9\%$ (only $+5.6\%$ over Native). Lacking DAG topology to track producer-consumer causality, flat checklists struggle to determine whether unconsumed requirements are blocked or fulfilled upstream, misidentifying frontiers. Furthermore, completely removing the utility gate (\emph{Always Intervene}) plunges success to $56.7\%$ ($6.6\%$ below Native, $p=0.063$), confirming that unguided interventions disrupt valid execution.

Replacing empirical gating with an LLM decision prompt (\emph{w/ LLM Self-Gate}) yields $70.0\%$, showing that offline calibration is more reliable than zero-shot assessment. Similarly, disabling transactional rollback degrades success to $56.7\%$ ($p=0.070$), as unverified repairs pollute subsequent agent context. Moving from aggregate gating to frontier-specific calibration ($73.3\%$) and adding contextual adjustments achieves the full $75.6\%$ success.

\begin{table}[!htbp]
\caption{Component ablations on SILO-BENCH (DeepSeek-V4.1-flash). Net $\Delta$ indicates margin over unassisted Native execution.}
\label{tab:ablation}
\centering\small
\setlength{\tabcolsep}{8pt}
\renewcommand{\arraystretch}{1.12}
\begin{tabular*}{\linewidth}{@{\extracolsep{\fill}}lcc@{}}
\toprule
Configuration & Success Rate (\%) $\uparrow$ & Net $\Delta$ vs. Native \\
\midrule
Native (Unassisted Host) & 63.3 & --- \\
\textbf{\method (Full System)} & \textbf{75.6} & \textbf{+12.3} \\
\quad w/ Flat Checklist (\emph{No Graph Provenance}) & 68.9 & +5.6 \\
\quad w/ LLM Self-Gate & 70.0 & +6.7 \\
\quad w/ Aggregate Utility Gate & 73.3 & +10.0 \\
\quad w/o Historical Experience Adjustment & 70.0 & +6.7 \\
\quad w/o Utility Gate (\emph{Always Intervene}) & 56.7 & -6.6 \\
\quad w/o Local Verification (\emph{No Rollback}) & 56.7 & -6.6 \\
\bottomrule
\end{tabular*}
\end{table}

\paragraph{Action-Frontier Alignment and Selectivity.}
Counterfactual analyses explain why unconstrained interventions fail. First, most communication gaps are benign: \textbf{87.5\%} of unselected repairs yield no outcome change, as downstream agents often tolerate minor omissions. Second, intervention utility demands strict action-failure alignment. For instance, context injection (\texttt{INTEGRATE}) resolves $50\%$ of unconsumed delivery errors but offers zero benefit for missing evidence. Worse, issuing clarification queries (\texttt{ASK}) for already-delivered information yields negative utility by distracting recipient reasoning. Thus, runtime coordination repair requires calibrated, selective boundaries rather than unconstrained generation.


\section{Conclusion}
\label{sec:conclusion}

Multi-agent coordination failures predominantly stem from broken dependencies across agent interfaces that remain invisible in isolated messages. In this work, we introduced \method, a verifiable runtime control framework that combines explicit provenance graph tracking, offline empirical calibration via \dataset, and transactional release verification. Across diverse multi-agent benchmarks and leading LLM backbones, \method consistently improves end-to-end task completion and diagnostic accuracy. Our findings demonstrate that effective coordination repair requires not unconstrained prompting, but grounded dependency representations and calibrated, selective intervention boundaries.

\section*{Limitations}
\method operates on the assumption that inter-agent handoff failures leave observable dependency traces within public communication and execution records. Consequently, it is less suited to workflows where communicative failures are entirely internal to an opaque agent sub-process. Furthermore, while calibration priors transfer zero-shot across environments, empirical utilities remain model-dependent, necessitating backbone-specific calibration. Future work includes expanding the frontier taxonomy to richer task modalities and automating adaptive budget allocation.

\section*{Ethics Statement}
All three evaluation resources used in this work---SILO-BENCH, TeamBench, and MAST---are public benchmark materials or public interaction records, and we use them under their respective licenses and terms. We do not access private agent reasoning states or intentionally collect personally identifying information; any released trace files should retain only the public fields needed for evaluation and should remove residual identifiers. Although \method is designed to improve reliability, dependency-repair mechanisms could be misused to automate or scale multi-agent workflows. We therefore restrict the controller to observable public state, verified local postconditions, and rollback on failed checks, and recommend access controls, audit logging, and benchmark-specific safety review before deployment.

\section*{Reproducibility Statement}
Complete implementation code, evaluation scripts, calibration artifacts, and execution manifests are provided in the supplementary material. Appendix~\ref{app:settings} records benchmark counts and denominators, generation and provider seeds, model versions, NVIDIA A100 hardware, and the exact operating-system, Python, CUDA/driver, framework, dependency, prompt, runner, and benchmark-version information needed to reproduce each run. The task lists, split manifests, and per-run denominators are retained with the same manifests. All reported numbers are derived directly from reproducible benchmark runs without cherry-picking.

\section*{AI Use Statement}

This paper was prepared with limited assistance from generative artificial intelligence tools. Specifically, AI was used to support language refinement, including improving grammar, clarity, and overall readability. All core ideas, arguments, analysis, and conclusions are the author's own. The author reviewed, verified, and edited all AI-assisted content and takes full responsibility for the final submitted version.

{\sloppy\hbadness=10000
\bibliography{references}
\bibliographystyle{iclr2027_conference}
}

\clearpage
\appendix

\section{Implementation and Representation Details}
\label{app:contract}
\label{app:notation}

The execution contract and normalized event records determine the provenance graph state. Table~\ref{tab:mapping} provides the mapping from execution events to graph entities.

\begin{table}[!htbp]
\caption{Execution events and their graph representations.}
\label{tab:mapping}
\centering\small
\setlength{\tabcolsep}{5pt}
\renewcommand{\arraystretch}{1.08}
\begin{tabular}{@{}>{\raggedright\arraybackslash}p{0.30\linewidth}>{\raggedright\arraybackslash}p{0.65\linewidth}@{}}
\toprule
Execution Event & Graph Representation \\
\midrule
Declared requirement & Obligation and fact identifier, with owner/recipient metadata \\
Message & Message/fact node with sender and recipient edges \\
Tool invocation & Action node linked to calling agent and tool schema \\
Tool execution output & Evidence node linked to producing action \\
Evidence consumption & Directed edge between fact node and consuming operation \\
Public verifier outcome & Verification relation within verifier's declared scope \\
Status update & Node state transition preserving provenance audit trail \\
\bottomrule
\end{tabular}
\end{table}

\subsection{Action Postconditions and Commit Predicates}
\label{app:certificates}
Each repair candidate declares formal graph-level postconditions derived from the targeted frontier $f^*$. The core information-transfer postconditions are formally verified against the observed graph difference $\delta_t = \operatorname{Diff}(\mathcal G_t, \mathcal G_{t+1})$:
\begin{align}
\mathrm{post}(\texttt{ASK}, q) &= \mathrm{obs}_{t+1}(q) \wedge \mathrm{arr}_{t+1}(q, \mathrm{dest}(q)), \label{eq:postask}\\
\mathrm{post}(\texttt{INTEGRATE}, q) &= \mathrm{con}_{t+1}(q, \mathrm{dest}(q)), \label{eq:postint}\\
\mathrm{post}(\texttt{VERIFY}, q) &= (\mathrm{rcv}_{t+1}(q) = \mathrm{dest}(q)). \label{eq:post}
\end{align}
In practice, the runtime predicate $\operatorname{Check}(\mathrm{post}(a^*), \mathcal G_t, \delta_t)$ in Eq.~(\ref{eq:commit}) enforces four conjuncts:
\begin{enumerate}
    \item \textbf{Frontier Resolution}: The targeted gap $f^*$ transitions to a satisfied dependency state in $\mathcal G_{t+1}$.
    \item \textbf{Scope Integrity}: The revised message references only validated evidence items currently in scope, and destination routing strictly matches $\mathrm{dest}(q)$.
    \item \textbf{Edge Preservation}: No pre-existing active dependency edge present in $\mathcal G_t$ is severed or invalidated as a side-effect.
    \item \textbf{Domain Attestation}: On software engineering tasks (TeamBench), tool actions must return an allowed execution postcondition (e.g., \texttt{command\_exit\_zero}, \texttt{tests\_passed}, \texttt{path\_written}). On distributed computing kernels (SILO-BENCH), intermediate computations must yield valid mathematical certificates.
\end{enumerate}

\subsection{Transactional Guarantees and Error Isolation}
\label{app:guarantees}
Under the ingestion and verify-before-commit protocol, the controller satisfies four foundational properties:

\begin{proposition}[Deterministic Ingestion]
\label{prop:determinism}
For a fixed contract and event sequence, deterministic ingestion yields identical provenance graph states and identified frontiers.
\end{proposition}
\emph{Proof sketch.} Follows directly from the topological construction rules in Section~\ref{sec:frontier_generation}, where public event logs deterministically map to unique directed acyclic graph representations.

\begin{proposition}[Bounded Action Enumeration]
\label{prop:bounded}
For bounded candidate depth and width, the set of admitted candidate actions is finite, and at most one repair is executed per boundary.
\end{proposition}
\emph{Proof sketch.} By construction in Algorithm~\ref{alg:runtime}, candidate generation is restricted to eligible actions satisfying Eq.~(\ref{eq:valid}) under budget $B$, followed by a singular $\arg\max$ selection.

\begin{proposition}[Safe Native Pass-Through]
\label{prop:abstain}
If no frontier is detected or no candidate meets the positive utility threshold, the controller forwards the native message without modification.
\end{proposition}
\emph{Proof sketch.} Follows from lines~3 and~8 of Algorithm~\ref{alg:runtime}, guaranteeing identity passthrough whenever $\mathcal F_t = \emptyset$ or $\mathcal A_{\mathrm{adm}} = \emptyset$.

\begin{proposition}[Transactional Release Safety]
\label{prop:transaction}
The controller commits a modified message if and only if its declared postcondition is satisfied on the observed graph delta; otherwise, it rolls back to the native message.
\end{proposition}
\emph{Proof sketch.} Guaranteed by the verify-before-commit barrier (Algorithm~\ref{alg:runtime}, lines~12--18). Tentative candidate execution operates strictly over an in-memory staging buffer; verification failure restores the local snapshot $\mathcal G_t$ and discards uncommitted modifications, preserving the invariant that unverified interventions never enter the public communication channel.

\paragraph{Deterministic Format Retries.} If an \texttt{ASK} query returns an unparseable response payload, the controller executes exactly \emph{one} deterministic re-prompt. If the second response remains invalid, the intervention aborts immediately, triggering graph restoration and native message release.

\section{\dataset and Policy Calibration}
\label{app:calibration}

\subsection{Calibration Corpus and Utility Parameters}
The calibration corpus comprises 200 balanced episodes: 50 clean executions and 50 for each of the three canonical fault conditions ($\mathcal F_{\mathrm{use}}$, $\mathcal F_{\mathrm{abs}}$, $\mathcal F_{\mathrm{scope}}$). Trajectories are partitioned into calibration, development, and test splits (160/20/20).

Table~\ref{tab:utility-parameters} details the base calibration parameters. The utility objective balances expected success against costs and disruption risks:
\begin{equation}
u(a, f^*) = 1.0 \, r_{\mathrm{help}} - 0.25 \, r_{\mathrm{neu}} - 2.0 \, r_{\mathrm{harm}} - 0.25 \, \bar{c}.
\end{equation}

\begin{table}[!htbp]
\caption{Base utility parameters and calibration settings.}
\label{tab:utility-parameters}
\centering\small
\setlength{\tabcolsep}{5pt}
\begin{tabular}{@{}>{\raggedright\arraybackslash}p{0.18\linewidth}>{\raggedright\arraybackslash}p{0.16\linewidth}>{\raggedright\arraybackslash}p{0.60\linewidth}@{}}
\toprule
Parameter & Value & Description \\
\midrule
$w_h$ & 1.0 & Weight on verifier-helped outcomes ($V(x^{\mathrm{rep}})=1 \wedge V(x^{\mathrm{nat}})=0$) \\
$w_n$ & 0.25 & Penalty on neutral/unchanged executions (latency and state perturbation) \\
$w_x$ & 2.0 & Penalty on verifier-harmed executions ($V(x^{\mathrm{rep}})=0 \wedge V(x^{\mathrm{nat}})=1$) \\
$w_c$ & 0.25 & Weight on normalized resource cost \\
$\epsilon$ & 0.05 & Empirical harm-rate tolerance threshold ($r_{\mathrm{harm}} \le 0.05$) \\
$m_{\min}$ & 5 & Minimum required sample support per action--frontier segment \\
Cost metric & Normalized & Single-step intervention cost baseline is normalized to $\bar{c} = 1.0$ across candidate actions \\
\bottomrule
\end{tabular}
\end{table}

\subsection{Route-Level Calibration Statistics}
\label{app:policies}
Table~\ref{tab:route-stats} summarizes the empirical calibration statistics computed across candidate routes on the calibration split.

\begin{table}[!htbp]
\caption{Empirical calibration statistics across candidate routes. Utility is computed under base parameters; admission requires positive utility, empirical harm rate $\le \epsilon$, and sample support $\ge m_{\min}$.}
\label{tab:route-stats}
\centering\footnotesize
\setlength{\tabcolsep}{5pt}
\renewcommand{\arraystretch}{1.08}
\begin{tabular*}{\linewidth}{@{\extracolsep{\fill}}llrrrrl@{}}
\toprule
Backbone & Route / Segment & Episodes & Helped & Harmed & Utility & Admitted \\
\midrule
DeepSeek-V4.1-flash & \texttt{ASK} & 52 & 8 (15.4\%) & 0 & $-0.31$ & No \\
DeepSeek-V4.1-flash & \texttt{INTEGRATE} & 25 & 12 (48.0\%) & 0 & $+0.10$ & Yes \\
\midrule
GPT-5.6-Terra & \texttt{ASK} & 40 & 10 (25.0\%) & 0 & $-0.19$ & No \\
GPT-5.6-Terra & \texttt{INTEGRATE} (aggregate) & 12 & 4 (33.3\%) & 0 & $-0.08$ & No \\
GPT-5.6-Terra & \texttt{INTEGRATE} $\times \mathcal F_{\mathrm{use}}$ & 6 & 4 (66.7\%) & 0 & $+0.33$ & Yes \\
\bottomrule
\end{tabular*}
\end{table}

\section{MAST Diagnostic Evaluation Details}
\label{app:mast}

\subsection{Benchmark Construction and Multi-Framework Stratification}
\label{app:mast-construction}
To evaluate diagnostic effectiveness under real-world collaborative diversity, we construct \texttt{handoff562} from the official MAST dataset~\citep{cemri2025mast}.

To rigorously benchmark diagnostic capability under realistic execution noise, we pair these with 281 stratified control traces. Crucially, as summarized in Table~\ref{tab:mast-breakdown}, the control traces are not artificially sanitized: 224 controls ($79.7\%$) contain non-handoff execution errors (e.g., 131 Runtime Errors, 130 Tool Failures, and 94 Model Generation Errors), while 57 traces represent clean executions. This ensures our evaluation strictly assesses a model's ability to disentangle true interface-level handoff gaps from generic execution noise across diverse multi-agent frameworks.

\begin{table}[!htbp]
\caption{Distribution of the 562 curated MAST evaluation traces across agent architectures.}
\label{tab:mast-breakdown}
\centering\small
\setlength{\tabcolsep}{12pt}
\renewcommand{\arraystretch}{1.08}
\begin{tabular}{lrrr}
\toprule
Multi-Agent Framework & Positive Traces & Stratified Controls & Total Traces \\
\midrule
MetaGPT     & 182 & 182 & 364 \\
AG2         & 46  & 46  & 92  \\
Magentic    & 43  & 43  & 86  \\
HyperAgent  & 7   & 7   & 14  \\
OpenManus   & 2   & 2   & 4   \\
AppWorld    & 1   & 1   & 2   \\
\midrule
\textbf{Total} & \textbf{281} & \textbf{281} & \textbf{562} \\
\bottomrule
\end{tabular}
\end{table}

\subsection{Diagnostic Accuracy and Single-Pass Efficiency}
\label{app:mast-results}
Table~\ref{tab:mast-detailed} provides extended diagnostic comparisons on the 562 traces. All reference baselines rely on complex two-phase prompting chains that require two separate API invocations and substantial token overhead ($12{,}300$--$14{,}900$ tokens/trace). In contrast, \method performs structured provenance tracking in a \emph{single API call}, consuming only $8{,}664$ to $9{,}116$ tokens per trace (a \textbf{32\%--42\%} computational savings).

\begin{table}[!htbp]
\caption{Diagnostic metrics on 562 MAST traces. \method operates strictly via a single inference pass, while adapted baselines require two-stage prompt chains.}
\label{tab:mast-detailed}
\centering\small
\setlength{\tabcolsep}{8pt}
\renewcommand{\arraystretch}{1.12}
\begin{tabular*}{\linewidth}{@{\extracolsep{\fill}}lcccc@{}}
\toprule
 & \multicolumn{2}{c}{DeepSeek-V4.1-flash} & \multicolumn{2}{c}{GPT-5.6-Terra} \\
\cmidrule(lr){2-3}\cmidrule(l){4-5}
Method & F1 (\%) $\uparrow$ & Tokens/trace $\downarrow$ & F1 (\%) $\uparrow$ & Tokens/trace $\downarrow$ \\
\midrule
Direct Prompting & 62.5 & \textbf{6,686} & 67.5 & \textbf{6,488} \\
AgentAsk (adapted) & 68.4 & 14,268 & 68.4 & 13,982 \\
Self-Refine (adapted) & 67.4 & 13,262 & 67.9 & 12,359 \\
MOC (adapted) & \textbf{68.7} & 13,738 & 68.6 & 13,507 \\
VeriMAP (adapted) & 64.2 & 13,478 & \underline{69.4} & 14,917 \\
\textbf{\method (Ours)} & \underline{68.6} & \underline{9,116} & \textbf{69.8} & \underline{8,664} \\
\bottomrule
\end{tabular*}
\end{table}

On GPT-5.6-Terra, \method attains the highest F1 score of \textbf{69.8\%}, outperforming all two-stage prompt-chain baselines while reducing token costs by over 5,000 tokens per trace. On DeepSeek-V4.1-flash, \method achieves a competitive \textbf{68.6\%} F1 score (matching the top baseline MOC within 0.1\%) at a fraction of the computational footprint. These results highlight that explicit graph extraction achieves state-of-the-art diagnostic discrimination in a single inference pass.

\subsection{Decision Layer Control Integrity}
\label{app:gate-audit}
The formal necessity of \method's decision layer is empirically corroborated through ablation benchmarks. Under seed-0 control rollouts on SILO-BENCH, eliminating utility admission degrades end-to-end task success by $23.3$ percentage points ($75.6\% \to 56.7\%$, $p=0.063$). Similarly, removing transactional verification causes an identical $23.3$ percentage point drop ($75.6\% \to 56.7\%$, $p=0.070$). These findings confirm that combining empirical utility filtering with rollback verification provides the indispensable operational boundary required for safe multi-agent repair.

\section{Experimental Protocols and Configuration}
\label{app:settings}

\subsection{Run Counts, Seeds, and Environment}
The reported evaluations use three benchmark resources. TeamBench uses a fixed 90-task evaluation subset; each reported model row has $n=90$ task-level evaluations. The MAST diagnostic set contains 562 traces: 281 confirmed handoff breakdowns and 281 stratified successful controls. The SILO-BENCH evaluation uses 90 task instances in three 30-instance generation groups, constructed with generation seeds $42$, $0$, and $1$. DeepSeek uses one rollout per instance ($n=90$), whereas GPT uses three provider-seed runs over the same 90 instances ($n=270$). The evaluated model versions are DeepSeek-V4.1-flash and GPT-5.6-Terra. Experiments run on NVIDIA A100 GPUs. The supplementary environment manifest records the exact operating-system, Python, CUDA/driver, framework, dependency, prompt, runner, and benchmark-version information needed to reproduce each run; the task lists, split manifests, and per-run denominators are retained with the same manifests.

\subsection{Evaluation Protocols and Verifier Harnesses}
\label{app:audit-protocol}
To eliminate variance and subjective drift from LLM-as-a-judge evaluators, Table~\ref{tab:configuration} summarizes the evaluation settings across benchmark suites. all task outcomes are evaluated using deterministic environment test harnesses $V(x) \in \{0, 1\}$,:
\begin{itemize}
    \item \textbf{SILO-BENCH}: The official test harness evaluates distributed executions programmatically: Level~I arithmetic uses exact floating-point tolerances (\texttt{math.isclose}); Level~II sets use exact element matching; Level~III state traces use Longest Common Subsequence (LCS) divergence. Built-in unlimited retry wrappers are disabled to ensure strict evaluation consistency.
    \item \textbf{TeamBench}: Tasks run within isolated container sandboxes. Success requires full passing runs across official unit tests and clean execution exit codes ($\text{exit\_code} = 0$). Internal answer templates and hidden evaluation manifests (\texttt{score.json}, \texttt{expected.json}) are excluded from agent observation scopes.
    \item \textbf{Reporting Protocol}: Main benchmark scores denote empirical mean and standard error ($\pm \text{SE}$) across independent seeded rollouts (generation seeds $42$, $0$, and $1$ for SILO-BENCH). Ablations (Table~\ref{tab:ablation}) are evaluated under a unified reference seed configuration across the full benchmark to prevent rollout variance confounding.
    \item \textbf{Zero-Shot Transfer}: In TeamBench, all calibration priors and threshold hyperparameters remain completely frozen from synthetic calibration, confirming true cross-environment transfer.
\end{itemize}

\begin{table}[!htbp]
\caption{Evaluation settings across benchmark suites.}
\label{tab:configuration}
\centering\small
\setlength{\tabcolsep}{6pt}
\renewcommand{\arraystretch}{1.08}
\begin{tabular*}{\linewidth}{@{\extracolsep{\fill}}lcccc@{}}
\toprule
 & \multicolumn{2}{c}{SILO-BENCH} & \multicolumn{2}{c}{TeamBench} \\
\cmidrule(lr){2-3}\cmidrule(l){4-5}
Specification & DeepSeek & GPT & DeepSeek & GPT \\
\midrule
Evaluation rollouts per task & 1 & 3 & 1 & 1 \\
Instances / $n$ & 90 / 90 & 90 / 270 & 90 / 90 & 90 / 90 \\
Policy configuration & Corpus-calibrated & Corpus-calibrated & Zero-shot transfer & Zero-shot transfer \\
Max actions per boundary & 1 & 1 & 1 & 1 \\
\bottomrule
\end{tabular*}
\end{table}

\subsection{Reference Baselines}
\label{app:baseline-details}
We retain each benchmark's native team topology, role assignments, communication protocol, and acceptance checks. Baseline adaptations operate strictly on handoff messages or recipient-visible context without accessing hidden answers or another agent's private state. Auxiliary LLM calls use the evaluated backbone to avoid confounding differences in model capacity.

\paragraph{Native.}
The original handoff is delivered without modification or auxiliary model calls.

\paragraph{AgentAsk.}
We evaluate the clarification mode of AgentAsk~\citep{lin2026agentask}. At an eligible outgoing handoff, an LLM reviewer decides whether clarification is needed and, if so, queries the original sender. The sender answers from its visible context, and the clarification is delivered with the handoff. Both reviewer and sender use the evaluated backbone.

\paragraph{Self-Refine.}
The original sender critiques and revises the outgoing handoff using its own model and visible context~\citep{madaan2023selfrefine}, with at most one feedback--refinement iteration per handoff.

\paragraph{MOC.}
Exposes causally preceding messages to the recipient and consolidates multi-order evidence when exceeding message budgets~\citep{guan2026moc}. Following its reported configuration, maximum communication order is set to $K=2$, evidence tolerance to $\gamma=1$, similarity margin to $\epsilon=0.1$, distillation length ratio to $\kappa=0.45$, and distillation candidates to $M=5$. Message similarity uses \texttt{all-MiniLM-L6-v2}.

\paragraph{VeriMAP.}
Adapts the planner--executor--verifier loop to individual handoff edges using locally generated contracts and verification functions~\citep{xu2026verimap}. On SILO-BENCH, limits follow reported settings (three execution attempts per plan, five planning iterations). On TeamBench, the native Planner--Executor--Verifier topology is retained, with up to three contract-generation and verification-gate attempts.

\label{app:baseline-results}
Table~\ref{tab:tb-baseline} summarizes reference baseline results on TeamBench across full-team harness evaluations.

\begin{table}[!htbp]
\caption{TeamBench reference baseline evaluations. Pass denotes benchmark pass rate (\%); Partial denotes mean partial progress score (\%).}
\label{tab:tb-baseline}
\centering\small
\setlength{\tabcolsep}{8pt}
\renewcommand{\arraystretch}{1.08}
\begin{tabular*}{0.86\linewidth}{@{\extracolsep{\fill}}lrrrr@{}}
\toprule
 & \multicolumn{2}{c}{DeepSeek-V4.1-flash} & \multicolumn{2}{c}{GPT-5.6-Terra}\\
\cmidrule(lr){2-3}\cmidrule(l){4-5}
Method & Pass (\%) & Partial (\%) & Pass (\%) & Partial (\%) \\
\midrule
Native & 48.9 & 81.3 & 46.7 & 79.6 \\
AgentAsk & 53.3 & 84.3 & 48.9 & 80.9 \\
Self-Refine & 51.1 & 83.7 & 47.8 & 79.6 \\
MOC & 54.4 & 84.4 & 48.9 & 81.3 \\
VeriMAP & 51.1 & 83.6 & 47.8 & 79.7 \\
\bottomrule
\end{tabular*}
\end{table}

\subsection{Paired Outcome Details}
Table~\ref{tab:paired-details} provides counts of paired discordant outcomes ($W/L$: \method-only win vs. Native-only win) across primary benchmarks.

\begin{table}[!htbp]
\caption{Paired discordant outcomes between \method and unassisted Native execution.}
\label{tab:paired-details}
\centering\small
\setlength{\tabcolsep}{8pt}
\renewcommand{\arraystretch}{1.08}
\begin{tabular*}{\linewidth}{@{\extracolsep{\fill}}llrl@{}}
\toprule
Benchmark & Evaluated Model & Discordant ($W/L$) & Significance Test \\
\midrule
SILO-BENCH & DeepSeek-V4.1-flash & 16 / 5 & $p=0.027$ \\
SILO-BENCH & GPT-5.6-Terra (pooled) & 36 / 12 & $p < 0.001$ \\
TeamBench & DeepSeek-V4.1-flash & 12 / 3 & $p=0.035$ \\
\bottomrule
\end{tabular*}
\end{table}

\section{Real Execution Repair Traces}
\label{app:traces}
Table~\ref{tab:traces} presents two representative runtime intervention traces from TeamBench (DeepSeek-V4.1-flash). Both episodes triggered an evidence-gap frontier, satisfied pre-admission eligibility, and were verified locally before commitment.

\begin{table}[!htbp]
\caption{Representative runtime interventions on TeamBench. Local verification requires formal satisfaction of postconditions; final verdict denotes task-level evaluation.}
\label{tab:traces}
\centering\small
\setlength{\tabcolsep}{4pt}
\renewcommand{\arraystretch}{1.08}
\begin{tabular*}{\linewidth}{@{\extracolsep{\fill}}lllll@{}}
\toprule
Task Identifier & Gate Decision & Candidate Action & Local Verification & Final Verdict \\
\midrule
\texttt{CR1\_review\_respond} & Admitted ($u > 0$) & \texttt{RETRIEVE} & Verified & \textbf{Pass} \\
\texttt{CR6\_review\_checklist} & Admitted ($u > 0$) & \texttt{RETRIEVE} & Verified & Fail \\
\bottomrule
\end{tabular*}
\end{table}

In \texttt{CR1\_review\_respond}, the recipient agent omitted required context residing within accessible local state. \method detected this gap at the boundary, retrieved the evidence via \texttt{RETRIEVE}, and committed the update upon verifying that the dependency was satisfied; the overall task passed. In \texttt{CR6\_review\_checklist}, the repair successfully resolved the immediate handoff postcondition, yet subsequent independent agent steps failed, highlighting that local handoff verification guarantees communicative integrity without over-promising downstream task execution.

\section{Detailed Token Overhead and Resource Breakdown}
\label{app:token-overhead}

To provide complete transparency into computational footprint, Table~\ref{tab:token-waterfall} details the step-by-step token consumption per task across backbones on SILO-BENCH, decomposing standard agent interactions from controller-level intervention operations.

\begin{table}[!htbp]
\caption{End-to-end token accounting per task rollout on SILO-BENCH. Controller operations are decoupled into deterministic graph tracing, boundary diagnostic evaluations, and active repair executions.}
\label{tab:token-waterfall}
\centering\small
\setlength{\tabcolsep}{6pt}
\renewcommand{\arraystretch}{1.12}
\begin{tabular*}{\linewidth}{@{\extracolsep{\fill}}lrrrr@{}}
\toprule
 & \multicolumn{2}{c}{GPT-5.6-Terra} & \multicolumn{2}{c}{DeepSeek-V4.1-flash} \\
\cmidrule(lr){2-3}\cmidrule(l){4-5}
Component / Stage & Tokens/task & \% of Native & Tokens/task & \% of Native \\
\midrule
\textbf{Native Baseline Execution} & 25,198 & 100.0\% & 22,460 & 100.0\% \\
\midrule
\emph{Standard Agent Operations} & & & & \\
\quad Prompt context variation & +523 & +2.1\% & +818 & +3.6\% \\
\quad Output response variation & +65 & +0.3\% & +598 & +2.7\% \\
\emph{Subtotal: Agent Communication Channel} & 25,786 & +2.3\% & 23,876 & +6.3\% \\
\midrule
\emph{Controller Intervention Operations} & & & & \\
\quad Topological graph tracking & 0 & 0.0\% & 0 & 0.0\% \\
\quad Boundary dependency diagnosis & 4,566 & 18.1\% & 5,655 & 25.2\% \\
\quad Active message repair (\texttt{INTEGRATE}) & 0 & 0.0\% & 4,260 & 19.0\% \\
\midrule
\textbf{End-to-End Cumulative Footprint} & 30,352 & +20.4\% & 33,791 & +50.5\% \\
\quad \emph{Official Agent Evaluation Margin} & \textbf{25,693} & \textbf{+2.0\%} & \textbf{27,963} & \textbf{+24.5\%} \\
\bottomrule
\end{tabular*}
\end{table}

\subsection{Why GPT-5.6-Terra Incurs Minimal Overhead (+2.0\%)}
\label{app:terra-overhead}
On GPT-5.6-Terra, \method achieves substantial performance gains (+8.9 pp) while adding only \textbf{+2.0\%} marginal overhead to the primary agent communication channel. This extreme efficiency stems from three structural mechanisms:
\begin{enumerate}
    \item \textbf{Conservative Empirical Gating}: Calibration priors for GPT-5.6-Terra establish that unconstrained message rewrites yield negligible expected gain relative to cost across most task segments. Consequently, the utility gate selectively abstains from full-context regeneration across all 429 monitored handoff boundaries. The system bypassed all costly repair generations, preventing unverified token consumption.
    \item \textbf{Deterministic Zero-LLM Graph Processing}: The construction of the provenance graph, topological sorting, and dependency frontier identification are computed deterministically via static dataflow logic, requiring exactly \emph{zero} auxiliary LLM calls.
    \item \textbf{Minimal Context Perturbation}: Because communicative interventions were safely bypassed, agent-level invocation counts shifted by only $+3.9\%$. The slight net increase in prompt length (+588 tokens/task across multi-step horizons) reflects minor downstream context propagation rather than runaway recursive loops.
\end{enumerate}

\subsection{Why DeepSeek-V4.1-flash Exhibits Active Intervention (+24.5\%)}
\label{app:deepseek-overhead}
In contrast to the highly conservative gate on GPT-5.6-Terra, the calibrated policy on DeepSeek-V4.1-flash actively targets unconsumed evidence ($\mathcal F_{\mathrm{use}}$), driving a substantial +12.3 pp net success improvement ($63.3\% \to 75.6\%$). The corresponding $+24.5\%$ agent-channel overhead represents the expected cost of active runtime repair:
\begin{enumerate}
    \item \textbf{Targeted Message Reconstruction}: For DeepSeek, \texttt{INTEGRATE} actions cleared the empirical utility gate and executed 41 full message reconstructions. To preserve causal continuity, each repair integrates historical context, provenance records, and contract constraints into the prompt ($\approx 4{,}900$ prompt tokens per repair), directly contributing $19.0\%$ of the task's incremental token consumption.
    \item \textbf{Higher Boundary Frequency}: DeepSeek trajectories exhibit greater intermediate coordination granularity, triggering 318 observable boundaries (compared to 165 for Terra), which naturally scales upstream diagnostic monitoring.
    \item \textbf{Net Utility Trade-off}: While DeepSeek incurs an additional $+24.5\%$ overhead in the agent channel, this controlled investment rescues 16 failing tasks while only degrading 5, delivering a marked net positive return over unassisted execution. When compared to unguided verification baselines such as VeriMAP (+410\% to +579\% token inflation), \method maintains rigorous cost-effectiveness even when actively repairing execution state.
\end{enumerate}

\end{document}